\documentclass{article}

\usepackage{amsmath,amsthm,amssymb,amsfonts}
\usepackage{graphicx}
\usepackage{arxiv}

\usepackage[utf8]{inputenc} 
\usepackage[T1]{fontenc}    
\usepackage{hyperref}       
\usepackage{url}            
\usepackage{booktabs}       
\usepackage{amsfonts}       
\usepackage{nicefrac}       
\usepackage{microtype}      
\usepackage{lipsum}

\usepackage{graphicx, subcaption}

\usepackage[ruled,vlined]{algorithm2e}
\SetKwFor{For}{for (}{) $\lbrace$}{$\rbrace$}

\newcommand{\R}{\mathbb{R}}
\newcommand{\N}{\mathbb{N}}
\newcommand{\GN}{\mathcal{N}} 

\newcommand{\E}{\mathbb{E}}

\newcommand{\M}{\mathcal{M}}
\newcommand{\D}{\mathcal{D}}

\newcommand{\V}{\mathcal{V}}
\newcommand{\I}{\mathcal{I}}

\newcommand{\eqdef}{\overset{\mathrm{def}}{=\joinrel=}}

\theoremstyle{definition}
\newtheorem{definition}{Definition}[section]
\theoremstyle{assumption}
\newtheorem{assumption}{Assumption}[section]
\theoremstyle{proposition}

\theoremstyle{theorem}
\newtheorem{theorem}{Theorem}[section]
\theoremstyle{fact}

\theoremstyle{corollary}

\theoremstyle{lemma}
\newtheorem{lemma}{Lemma}[section]

\theoremstyle{remark}

\title{A Theoretical Perspective on Differentially Private Federated Multi-task Learning}

\author{
  Huiwen Wu \\
  Ant Group\\
  Hangzhou, Zhejiang, China \\
  \texttt{huiwen.whw@antfin.com} \\
   \And
Cen Chen \\
  Ant Group\\
 Hangzhou, Zhejiang, China \\
  \texttt{chencen.cc@antfin.com} \\
   \AND
  Li Wang\\
  Ant Group\\
  Hangzhou, Zhejiang, China \\
  \texttt{raymond.wangl@antfin.com}  \\
}

\begin{document}
\maketitle

\begin{abstract}
In the era of big data, the need to expand the amount of data through data sharing to improve model performance has become increasingly compelling. As a result, effective collaborative learning models need to be developed with respect to both privacy and utility concerns. In this work, we propose a new federated multi-task learning method for effective parameter transfer with differential privacy to protect gradients at the client level. Specifically, the lower layers of the networks are shared across all clients to capture transferable feature representation, while top layers of the network are task-specific for on-client personalization. Our proposed algorithm naturally resolves the statistical heterogeneity problem in federated networks. We are, to the best of knowledge, the first to provide both privacy and utility guarantees for such a proposed federated algorithm. The convergences are proved for the cases with Lipschitz smooth objective functions under the non-convex, convex, and strongly convex settings. Empirical experiment results on different datasets have been conducted to demonstrate the effectiveness of the proposed algorithm and verify the implications of the theoretical findings.
\end{abstract}


\section{Introduction}
In the era of big data, data quality and quantity have become the most important factors that affect the effectiveness of the machine learning models trained.
The need to expand the amount of data through data sharing to improve model performance has become increasingly compelling. 
However, in reality, data are always isolated in different data federation such as organizations, companies, or edge devices. Data privacy is difficult to be effectively guaranteed across these data federation. 
Thus, there is an increasing interest in jointly training machine learning models without sharing data. 

To address such ``data isolation" problem, Federated Learning (FL) was proposed as a decentralized approach that enables collaboratively training while keeping the data on clients by only exchanging gradients/model parameters~\cite{li2020federated}. In particular,
Federated Averaging (\texttt{FedAvg} \cite{mcmahan2017communication}) was proposed, as the de facto method in the federated optimization, by averaging the gradients from the local clients. 
However, even sharing gradients may unintentionally lead to information leakage~\cite{shokri2017membership,hayes2019logan,melis2019exploiting,zhu2019deep}.
In order to protect the gradients of local clients, several approaches have been explored. Cryptographic approaches based on homomorphic encryption and secret sharing to ensure the privacy of local information can be found in \cite{bonawitz2016practical, gao2019privacy}. 
Those methods are computationally inefficient for non-linear operations, thus not practical for ML models at large scale or trained with frequent communications.
Several recent studies address the privacy issue in FL by combining differential privacy with existing federated algorithms to provide privacy guarantees~\cite{geyer2017differentially,heikkila2017differentially, dziugaite2018data, bassily2019private}.

Aside from the privacy concern, the inherent non-IID issue is also challenging, as data from different clients can be arbitrarily heterogeneous \cite{khaled2019first}.
Although \texttt{FedAvg} has shown to be empirically effective in heterogeneous settings, it cannot fully address the fundamental statistical heterogeneity \cite{mcmahan2017communication}, as it does not have the flexibility of allowing variable amount of updates for different clients.
To help address the statistical heterogeneity, different methods for personalization have been proposed to adapt global models for individual clients~\cite{smith2017federated, chen2020fedhealth, kulkarni2020survey}. 
Most of them employ a two-stage approach, i.e., collaboratively training the global model followed by client-level personalization through transfer learning techniques, such as 
parameter fine-tuning \cite{wang2019federated,mansour2020three}, 
model alignment \cite{chen2020fedhealth}, 
and knowledge distillation \cite{li2019fedmd}. 
However, the application of such two-stage methods is limited, as the performance of the locally adapted local model may be limited by the global model which is solely optimized for global accuracy. 
Few recent FL works jointly learn global and local models in a multi-task fashion by regularizing local model in objective function with task covariance or distance metrics between local and global parameters~\cite{smith2017federated,li2018federated}. 
However, those methods may increase the risk of overfitting, as significantly more model parameters are introduced in FL settings.

\begin{figure}[t!]
\centering
\includegraphics[width=0.7 \linewidth]{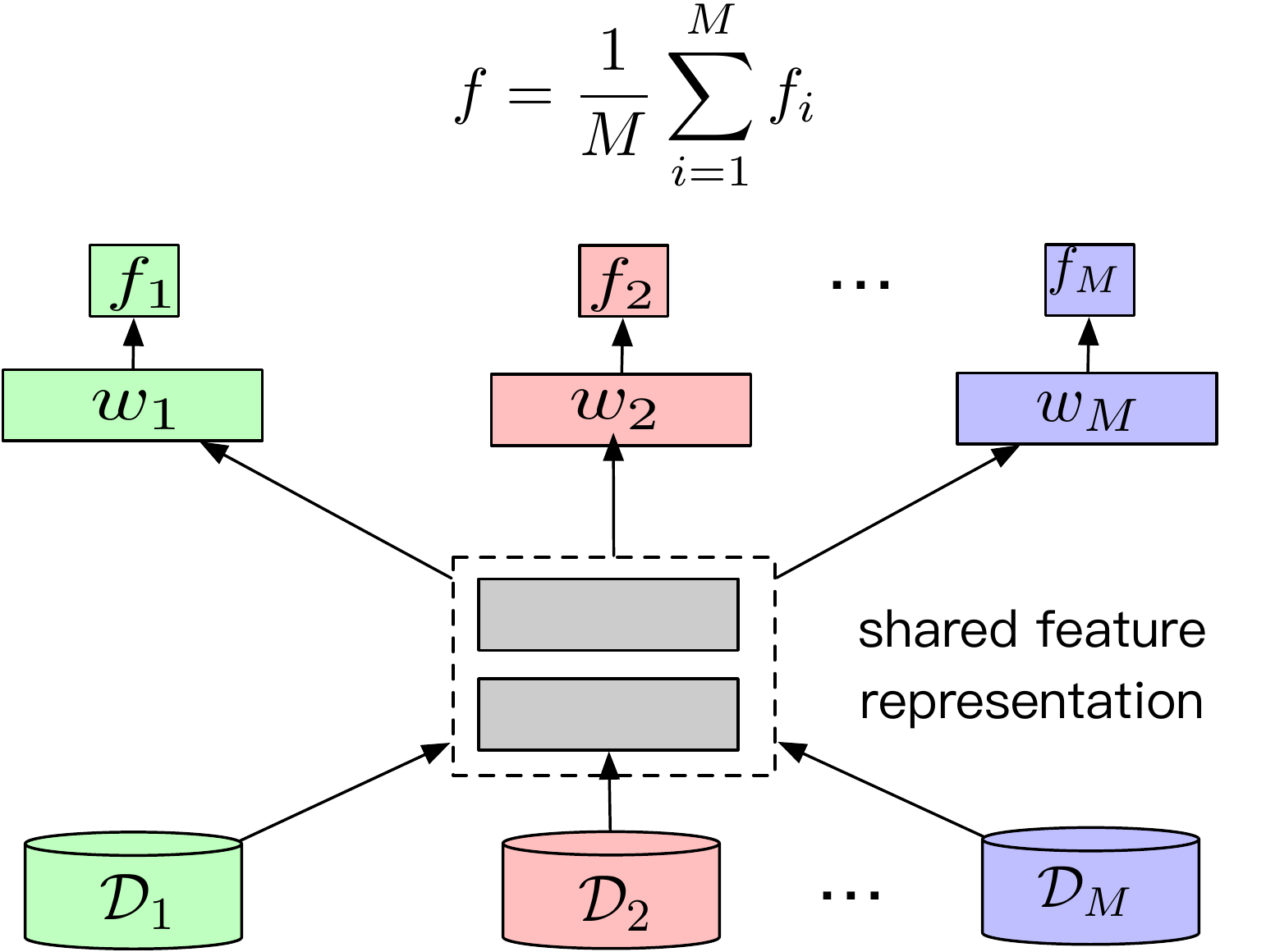}
\caption{Multi-Task Learning}
\label{fig:mtl}
\end{figure}

\begin{figure}[t!]
\centering
\includegraphics[width=0.7 \linewidth]{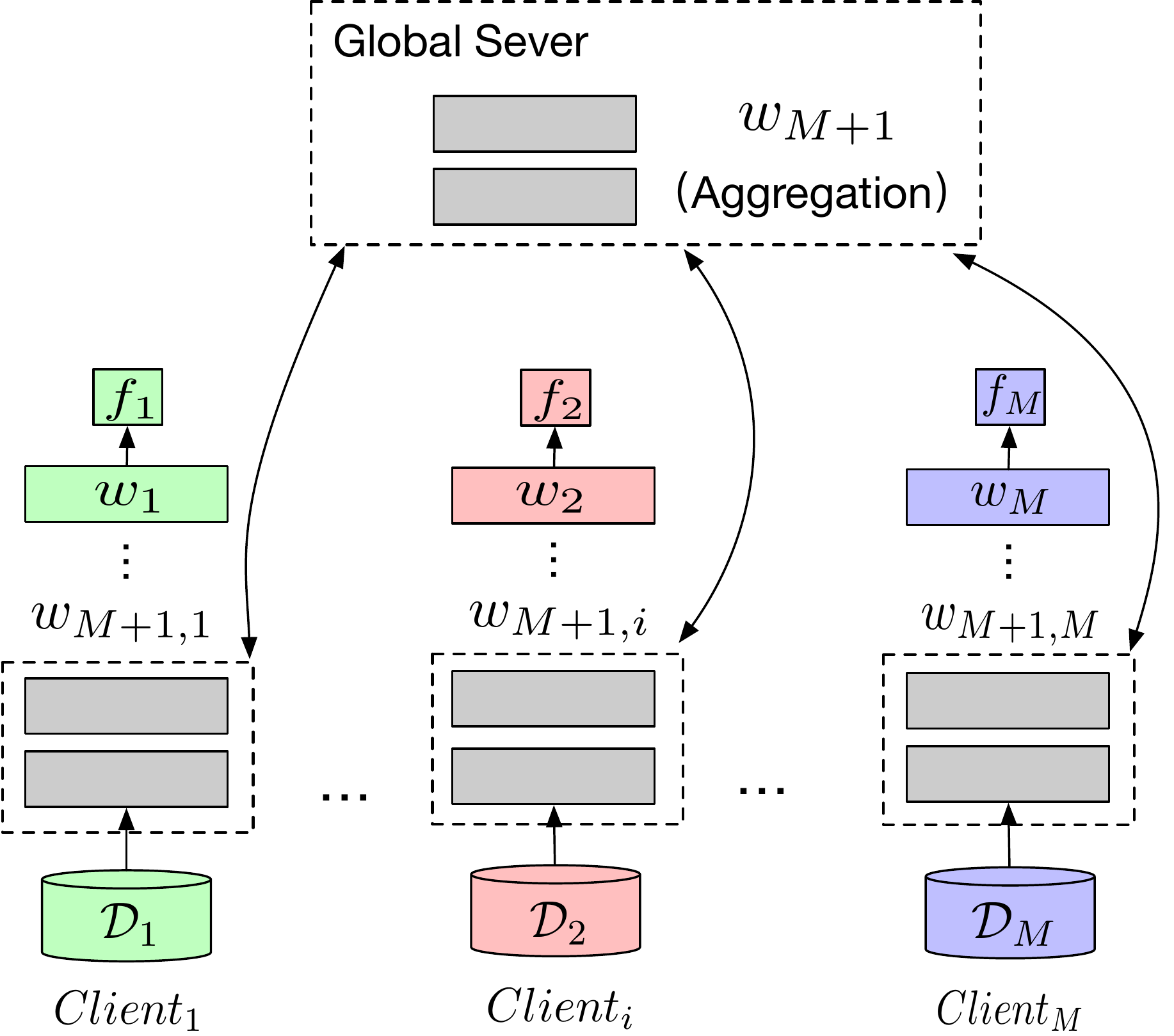}
\caption{Federated Multi-task Learning}
\label{fig:federated_mtl}
\end{figure}

In this paper, we present a Differentially Private Federated Multi-Task Learning method (\texttt{DPFedMTL}), that particularly enables federated optimization of heterogeneous client networks/tasks and protects the local model gradient information via differential privacy. 
More specifically, we first extend the widely used network structure in Multi-Task Learning (MTL) ~\cite{caruana1997multitask,zhang2017survey} (shown in Figure~\ref{fig:mtl}) that jointly learns a shared encoder to capture transferable feature representations across tasks by hard parameter sharing and utilizes task-specific upper layers to capture task heterogeneity, to the federated settings (shown in Figure~\ref{fig:federated_mtl}).
Our focus of this paper is more on investigating our proposed \texttt{DPFedMTL} algorithm \textit{from a theoretical perspective}. We provide a detailed analysis and proofs for both privacy guarantees and algorithm convergence. 
Our contributions can be summarized as follows:
\begin{itemize}
    \item We extend the widely used MTL paradigm with hard parameter sharing to federated settings to better model client heterogeneity.  We further propose to incorporate differential privacy at the client-side to protect gradient information during the FL  communications.
    \item We are, to the best of our knowledge, the first to provide both privacy and convergence guarantees for the proposed federated algorithm. The convergences are proved for local Lipschitz smooth objective functions under global non-convexity, convexity, and strongly convexity. 
    \item Experiments on different non-IID datasets haven been conducted to demonstrate the effectiveness of the proposed algorithm and verify the theoretical implications.
\end{itemize}

\section{Background and Related Work}

\noindent \textbf{Differentially Private FL.}
One of the most important problems in FL to address is to protect gradients from leaking sensitive client-level information. 
Several recent works address the privacy issue in federated learning by combining DP with existing federated algorithms to provide privacy guarantees.
\cite{geyer2017differentially} clips the gradients by norms of clients' updates and shows a minor loss in model utility. 
Some explore bayesian DP in federated learning\cite{heikkila2017differentially,triastcyn2019federated}, while \cite{liang2020exploring} adopts Laplacian smoothing DP. 
Nevertheless, none of them study DP in the case of partial parameter sharing and aggregation.
In our work, we use Gaussian DP ~\cite{dong2019gaussian} for privacy analysis due to its efficiency in privacy accounting and privacy protection is at the client-level as describe in ~\cite{geyer2017differentially}. The goal is to hide the client's contribution during training.

\noindent \textbf{Personalization.}
Another problem raised in FL is the statistical heterogeneity, i.e., local data are non-IID. 
Different methods for personalization have been proposed to adapt global models for individual clients. Most of them employ the two-stage approach, i.e., FL training followed by client-specific personalization by leveraging transfer learning \cite{chen2020fedhealth, smith2017federated},
such as fine-tuning \cite{wang2019federated,mansour2020three}, model alignment \cite{chen2020fedhealth}, and knowledge distillation \cite{li2019fedmd}. 
However, the performance of the locally adapted local model may be bounded by the global model performance which is solely optimized for global accuracy. 
A few recent works jointly learn global and local models and regularize the local model in objective function by task covariance or distance metrics between the local and global parameters~\cite{smith2017federated,li2018federated}. Those methods may increase the risk of overfitting, as significantly more model parameters are introduced.
In our work, we extend the commonly used MTL paradigm with hard parameter sharing to the FL setting. Specifically lower layers of the network are shared to transfer knowledge across clients, while upper layers of the network are task-specific for personalization.

\noindent \textbf{Convergence Analysis of FL.}
 Theoretical works have focused on convergence analysis of federated learning with local gradient-type updates.
In \cite{haddadpour2019convergence}, authors give linear convergence for local gradient descent methods on federated learning of bounded gradient diversity, smoothness, and $\mu-$ Polyak-Lojasiewicz (PL) conditions of local objective functions. 
Convergence analysis of FedAvg with partial device participation on non-IID data is presented in \cite{li2019convergence}.
However, the existing convergence works are not applicable for FL methods with partial parameter sharing ~\cite{haddadpour2019convergence, khaled2019first, li2019convergence}. 
Inspired by the subspace decomposition method evolves in the numerical analysis field~\cite{xu1992iterative,chen2020randomized}, we interpolate parameters of local clients to a high dimensional space, formulate a global optimization problem in this high dimension space, and solve it by subspace decomposition technique.
As far as we know, we are the first to apply such a technique in analyzing the convergence of the proposed DPFedMTL algorithm 
for non-convex, convex and strongly-convex cases.

\section{Federated Multi-task Learning}
\label{sec:methods}
In this section, we first formalize the Multi-Task Learning (MTL) in the federated setting, followed by presenting the threat model. We highlight the necessity of incorporating differential privacy for federated MTL. Finally, we present our proposed \texttt{DPFedMTL} algorithm.

\subsection{Multi-Task Learning}
\label{sec:mtl}
Multi-task learning (MTL) allows deep neural networks to leverage useful knowledge from multiple related tasks to help improve the performance of all the tasks. It has demonstrated its effectiveness over various tasks in NLP and CV. Refer to the survey for more details~\cite{zhang2017survey}.

We assume there is a set of $M$ clients, each with a learning tasks $\tau_i$, where $i \in \{1,2,..., M\}$.
Each contains a training dataset $\D_{i} = \cup_{j=1}^{N_i} \{ (x_{i,j}, y_{i,j}) \}$, where $N_i$ is the amount of data for task $\tau_i$. 
Tasks can be either homogeneous or heterogeneous.
As tasks are often related in that they all share a common underlying representation, one of the most commonly used approaches to MTL in neural networks is through hard parameter sharing that dates back to \cite{caruana1997multitask}. 
The basic idea is to learn \textit{shared feature representation} by jointly optimizing different tasks, which is often achieved by sharing the parameters of hidden layers~\cite{caruana1997multitask,zhang2017survey}. 
In this work, we establish our proposed algorithm based on this type of MTL approach.

As sketched in Figure~\ref{fig:mtl}, typically the top layers of the network are kept as task-specific, while the lower layers are shared across all tasks. 
We denote the parameters of task-specific networks as $\{ w_1, w_2,..., w_M \}$ and shared parameters as $w_{M+1}$.
The goal for each task is to lean a function $f_i$, while the global objective is to optimize $f= \frac{1}{M}\sum_{i=1}^{M} f_i$. 

\subsection{Federated MTL Setup}
\label{sec:federated_mtl}

We consider a standard federated setting that involves collaborative training across $M$ clients (e.g., users, organizations, or devices).
Unlike centralized optimization, the goal for federated learning is to learn a global model in a decentralized fashion without exposing each client data $D_i$. 

As illustrated in Figure~\ref{fig:federated_mtl}, each client performs local optimization with respect to its local objective function $f_i$ as:
\begin{eqnarray}\label{def:local_obj}
f_i(w_i, w_{M+1}) 
&=&  \frac{1}{N_i} \sum_{j=1}^{N_i} \ell( x_{i,j}, y _{i,j}, w_i, w_{M+1}),
\end{eqnarray}
where $w_i$ is task-specific parameters of client $i$, $w_{M+1}$ is parameters of shared common layers, $\ell$ is loss function.  

In federated setup, the global objective for the central model is composed by the average of local objectives:
\begin{equation}\label{def:global_obj}
f \eqdef \frac{1}{M} \sum_{i=1}^M f_i
\end{equation}
where $f_i$ is the local objective function defined in (\ref{def:local_obj}). 

During the federated optimization process, parameters $w_{M+1}$ are shared across all clients and synchronized every one or a few steps. At each communication round, the central server performs aggregation of the model updates from the local clients, updates the global model, and then distributes it back to all the clients.

\subsection{Privacy in Federated MTL}
In FedMTL, the potential information leakage only arises when the clients synchronize their learned shared network parameters with the global server. Thus the potential leakage profile contains all the intermediate $w_{M+1, i}$ that each client $i$ reveals to the global sever during the local training process and the aggregated $w_{M+1}$ that the server distributes back.
We assume a \textit{semi-honest threat model} that honest-but-curious participants follow exactly the computation protocols but may try to infer as much information as possible when interacting with the others. The potential adversary can be both the global server and the local clients. 

A common paradigm 
to prevent deep learning models from leaking sensitive information during the training process is to inject randomized noises, such as Gaussian noises, at the gradient level~\cite{abadi2016deep} using the idea of Differential Privacy(DP).
To protect privacy, in this paper, we explore such randomized technique in the context of federated MTL to protect clients' gradients contribution during model communications. 

%
\begin{algorithm}[t!]
\small
\SetAlgoLined
\textbf{Input:} 
$M$ clients, 
$K$ number of clients participated in local training, 
variance of Gaussian noise $\sigma$, 
number of rounds $T$, 
sensitivity of local gradients $S$, 
aggregation interval $H$,
step size $\eta$. \\
\textbf{Intialize:} $w_i$ for all $i \in [1,2,\cdots,M+1]$ \\
\For{$t := 0$ to $T-1$}{
	Server randomly sample a subset $\I_t$ of $K$ clients with distribution $\alpha_i = \frac{K}{M}$ with Poisson sampling.\\
	Server sends $w^{t}_{M+1}$ to all chosen clients. \\
	\textbf{Local Update:} \\
	\For{each chosen client $i \in I_t$}{
	    Synchronize parameters: $w^{t}_{M+1,i} \leftarrow  w^{t}_{M+1}$ \\
	    \textbf{For shared layers:} { \\
	    $\;\;\;$Compute gradients: $g_{M+1,i}^t = \partial_{w_{M+1}} f_i(\mathbf{w}^t)$  \\
		$\;\;\;$Perturb gradients: 
		\begin{equation} \label{update:perturb}
		\tilde{g}_{M+1,i}^t = {g}_{M+1, i}^t +n_{M+1,i}, ~ n_{M+1,i} \sim \GN(0,\sigma^2 S^2)
		\end{equation}  \\		 
		$\;\;\;$Update parameters: 
				\begin{equation}  \label{update:DPFTL_c}
        					w^{t+1}_{M+1,i} = w^{t}_{M+1,i} - \eta \left( \tilde{g}_{M+1,i}^t  \right)
				\end{equation}
	    } \\
	    \textbf{For task-specific layers:}{ \\
	    $\;\;\;$Compute gradients: $g_{i}^t = \partial_{w_i} f_i(\mathbf{w}^t)$ \\
		$\;\;\;$Update parameters: 	    
			\begin{equation} \label{update:DPFTL_l}
    				w^{t+1}_{i} = w^t_{i} - \eta g_i^t 
    			\end{equation}
	    }
%
	}	
	\textbf{Server Aggregation:}\\
	\If{$t~\rm{mod}~H == 0$}{
	Collect model weights $w^{t+1}_{M+1,i}$ sent by all clients. \\
	Aggregates the weights:
	\begin{equation}  \label{update:aggre}
	 w^{t+1}_{M+1} = \frac{1}{K} \sum_{i \in \I_t}  w^{t+1}_{M+1, i}
	 \end{equation}
	Send $w^{t+1}_{M+1}$ to each client.
	}
}
\caption{Differentially Private Federated MTL}
\label{alg:DPFedMTL}
\end{algorithm}

\subsection{Differentially Private Parameter Transfer}
Having the notations and terminology described above, we now present the algorithm of Differentially Private Federated Multi-Task Learning (\texttt{DPFedMTL}) in Algorithm \ref{alg:DPFedMTL}. In our federated MTL setting, common knowledge is shared across the clients through hard sharing of the lower layers, while upper layers remain to be task-specific to capture task-level information.
In this way, the statistical heterogeneity can be fundamentally better modeled.
To be specific, at each training step, clients are selected with Poisson uniform sampling $\alpha_i = K/M$, where $K$ is the number of selected clients and $M$ is the number of all clients, optimizes its local objective function, and contributes differentially privately to the shared parameters $w_{M+1}$. 

Specifically, in this work, we achieve\textit{ $C_p \left( G_{{1}/{\sigma}} \right)^{\otimes T}$-DP for each client using Gaussian mechanism}.
The magnitude of Gaussian noises added are calibrated with gradient sensitivity $S$, which will be elaborated in details in Section~\ref{sec:privacy}.


%
\section{Privacy Analysis}
\label{sec:privacy}
In the proposed \texttt{DPFedMTL} algorithm, we employ a Gaussian mechanism $\M$ that injects Gaussian noises calibrated by function sensitivity~\cite{dwork2014algorithmic}. 
We denote $\M : \R^N \times \V_{M+1} \rightarrow \V_{M+1}$ as follows:
\begin{equation} \label{def:mec}
\M \left(\D_i, w_{M+1}^t \right) = w_{M+1}^{t+1}. 
\end{equation}

To obtain the privacy guarantee of \texttt{DPFedMTL}, we first quantify the \textit{sensitivity} of local gradients for shared layers. Suppose gradients are clipped, the sensitivity is, therefore, upper bounded. Let $\| \cdot \|$ denote $\ell_2$ norm. We have the following analysis. 
\begin{lemma} [Sensitivity of local gradients with clipping] Suppose local gradients of the shared layers are clipped with constant $C$. The sensitivity $S$ of averaging in local gradients is $S = 2 C K/M $. 
\label{thm:sens_2}
\end{lemma}
\begin{proof}
$$
\|g^t_{D_i} - g^t_{D_i'} \|  \leq 2 C
$$
As one sample changes in whole datasets $\D$, gradient $g_i^t$ changes in aggregation step (\ref{update:aggre}). 
Thus, we have:
\begin{align*}
   \| K/M (\sum_{i\in \I_t} g_{i,D_i}^t -  g_{i,D_i'}^t) \| &= \| K/M ( g_{i,D_i}^t -  g_{i,D_i'}^t )\|    \\
   &\leq 2 C K/M
\end{align*}

\end{proof}
Based on the sensitivity results in Lemma~\ref{thm:sens_2}, we have privacy loss per iteration via Gaussian Differential Privacy (GDP)~\cite{dong2019gaussian}.  
\begin{lemma} [Privacy per iteration] Suppose local gradients of shared layers are clipped with constant $C$.  Let $n_{M+1}$ be noise sampled from Gaussian mechanism $\M $ with variance $\sigma^2 \cdot S^2$, where $S = 2 C \frac{K}{M}$ defined in Lemma \ref{thm:sens_2}. Then $\M$ satisfies 
 $G_{1/\sigma}$-DP,
where 
$
G_{1/\sigma}\left(\cdot \right) = \Phi(\Phi^{-1}(1 - \cdot) - 1/\sigma)
$
and $\Phi$ denotes the standard normal Cumulative Distribution Function. 
\end{lemma}
\begin{proof}
By Theorem 2.7 in Gaussian Differential Privacy \cite{dong2019gaussian}, the average gradient updating step is  $G_{1/\sigma}$-DP and the following gradient averaging step is deterministic. Thus, we conclude $\M$ is  $G_{1/\sigma}$-DP.
\end{proof}

\subsection{Privacy Accounting}
In this section, we analyze the accumulated privacy loss of Algorithm~\ref{alg:DPFedMTL}  by Central Limit Theorem (CLT)
with Gaussian Differential Privacy (GDP) \cite{dong2019gaussian}. 
GDP has demonstrated its superiority and efficiency in tractably analyzing subsampling and approximated composition of deferentially private algorithms compared to moment accountant \cite{abadi2016deep}. 

\subsubsection{Subsampling}
We first analyze the privacy amplification of subsampling. Specifically, in Algorithm~\ref{alg:DPFedMTL}, Poisson sampling is adopted, i.e., uniform sampling without replacement with probability $\alpha_i = p = K/M$.
\begin{definition}
For $y \in (-\infty, \infty)$, $g^*(y) = \sup_{-\infty < x < \infty} yx - g(x)$ is the convex conjugate of function $g$. 
\end{definition}
\begin{definition}
Define inverse function of $f$ as $f^{-1}(y) = \inf \{ t \in [0,1]: f(t) \leq \alpha \}$ for $\alpha \in [0,1]. $
\end{definition}
\begin{definition} 
For any $p \in [0,1]$, define the operator $C_p$ acting on trade-off functions as:
$$
C_p(f) := \min \{f_p, f_p^{-1} \}^{**},
$$
where $C_p$ is called as the $p$-sampling operator.  
\end{definition}

Following the previous work~\cite{dong2019gaussian}, we have privacy analysis of the composition of subsampled mechanism and $\M$. 
\begin{lemma} [Subsampling] $\M$ defined in Eq.(\ref{def:mec}) satisfies 
$G_{1/\sigma}$-DP. The composition of $\M$ and Poisson subsampling with probability $p = K/M$ is $C_p \left( G_{1/\sigma} \right)$-DP. 
\end{lemma}

\subsubsection{Composition}
By Composition Theorem~\cite{dong2019gaussian}, we have the following accumulated privacy loss of Algorithm~\ref{alg:DPFedMTL} after a number of training steps $T$.  \begin{theorem}
Given the sampling probability $alpha_i = K/M$ and the number of steps $T$,
Algorithm~\ref{alg:DPFedMTL} is $C_p \left( G_{{1}/{\sigma}} \right)^{\otimes T}$-DP.
\end{theorem}

According to Central Limit Theorem in ~\cite{bu2019deep}, we have approximated bound $G_{\mu}$ of privacy loss.
\begin{theorem}\label{thm:privacy_CLT}
Suppose Algorithm~\ref{alg:DPFedMTL} run with number of steps $T$ and uniform sampling without replacement with distribution $\alpha_i = K/M$, which satisfy $p \sqrt{T} \rightarrow \nu$. Then 
$C_p \left( G_{1/\sigma} \right)^{\otimes T} \rightarrow G_{\mu}$ uniformly as $T \rightarrow \infty$
where $\mu = \nu \cdot \sqrt{T (e^{1/\sigma^2} - 1)}. $
\end{theorem}

\section{Convergence Analysis}
In this section, we analyze the convergence of our proposed \texttt{DPFedMTL} (Algorithm \ref{alg:DPFedMTL}). 
To the best of knowledge,  we are the first to give a convergence analysis of \texttt{DPFedMTL}. 
Due to the inherent network structure with the shared and task-specific parameters in this federated setting, we describe the parameter space as a stable decomposition of local client parameters via \textit{subspace decomposition}. 
Such decomposition technique originally comes from multi-grid methods~\cite{xu1992iterative, stuben2001review} and has been applied to large-scale optimization problem recently~\cite{chen2020convergence,chen2020randomized}. Our proposed methods is analyzed with \textit{randomized subspace methods} \cite{chen2020randomized} from the optimization perspective.  

To be specific, we prove convergence of Algorithm \ref{alg:DPFedMTL} under the federated setting 
with a fixed learning rate and Poisson sampling one client participate in each update, i.e. $K=1, \alpha_i =p =  \frac{1}{M}$ in Algorithm \ref{alg:DPFedMTL}. 
Three types of convergence are analyzed for \textit{Lipschitz continuous} local objective functions with global \textit{non-convex}, \textit{convex}, and \textit{strongly convex} assumptions. 
Details of necessary assumptions for all proofs are presented in convergence results Section \ref{sec:convergence_res}.

\subsection{Subspace Decomposition}
In our federated algorithm, the global parameter space is aggregated by local objective parameter spaces. 
Denote $[M] = \{ 1, \cdots, M\}$ be an integer set. 
Let $w$ be the ensemble of local parameters $w_m, ~ m \in [M+1]$, where $w_{M+1} \in \mathcal V_{M+1}$ denotes the parameters shared by all clients and $w_i \in \mathcal V_i, ~ i \in [M]$ denote the each parameters held privately by local client $i$. Then we have the \textit{global parameter space} $\mathcal V$ as:
\begin{equation}
\mathbf{w} = \begin{bmatrix} w_1 & \cdots & w_{M} & w_{M+1} \end{bmatrix}^{\intercal}, 
\end{equation}
where $w_i \in \mathcal V_i$
and 
\begin{equation}
\mathcal V = \oplus_{i=1}^{M+1} \mathcal V_i.
\end{equation}





For each client, local model $\phi_i$ is parametrized by local task-specific layers and shared layers, i.e., $\phi_i = \left( w_i , w_{M+1} \right)$. 
Thus, we have the \textit{local parameter space} $\widetilde{\mathcal V}_i$ for local client model $\phi_i $ as:
$$
\widetilde{\mathcal V}_i = \mathcal V_i + \mathcal V_{M+1}, ~ \widetilde{\mathcal V}_i \subset \mathcal V. 
$$
Therefore, $\mathcal  V$ can be decomposed as a sum of subspaces $\widetilde{\mathcal V}_i, ~ i \in [M]$. 
\begin{equation} \label{def:decomp_new}
\V= \sum_{i=1}^M \widetilde{\V}_i. 
\end{equation}
Note here Eq.(\ref{def:decomp_new}) is not necessarily a direct sum nor orthogonal. The redundancy comes from shared layers $w_{M+1}.$

In order to connect global parameter space and local parameter space, we introduce the following \textit{restriction and interpolation operators}. 

Denote unit vector with $1$ on $i$-th position and $0$ others as $e_i = [0, \cdots, 1, \cdots, 0]^{\intercal}$, we have:
\begin{definition}\label{def:res_pro_operators}
Define $R_i: \widetilde{\V}_i \mapsto \V$ be the \textit{restriction} operator and $I_i: \V \mapsto \widetilde{\V}_i$ be the \textit{interpolation} operator. 
\begin{equation} \label{def:res_oper}
R_i 
=
\begin{bmatrix}
e_i^{\intercal} \\
\sqrt{1/M} e_{M+1}^{\intercal}
\end{bmatrix}
~ i \in [M]
\end{equation}
\begin{equation} \label{def:pro_oper}
I_i= R_i^{\intercal} = 
\begin{bmatrix}
e_i & \sqrt{1/M} e_{M+1}
\end{bmatrix}
~ i \in [M]
\end{equation}
\end{definition}
With these two operators, we can rewrite the global objective function in Eq.(\ref{def:global_obj}) as follows:

\begin{align}\label{eqn:obj_global_2}
f(w) 
= \frac{1}{M} \sum_{i=1}^M f_i(R_i I_i w).
\end{align}
The update formula is $w^{t+1} = w^t - \eta^t g^t, $
where 
\begin{itemize}
\item Perturbed gradient at iteration $t$ when client $i$ is chosen. 
$$
\tilde{g}_i^t = \langle \partial_{w_i} f(w^t) , e_i \rangle + \frac{1}{M} \langle \partial_{w_{M+1}} f(w^t) + n_{M+1,i}, e_{M+1} \rangle 
$$
\item 
Accumulated of $\tilde{g}_i^ t$ on $w^t$  over subset $\I_t$ is  
$
g^t = \sum_{i \in \mathcal{I}_t} \tilde{g}_i^t. 
$
\end{itemize}
Synchronization is given by $w^{t} = \sum_{i=1}^M R_i I_i v_i^t,  $
where $v_i^t  = \langle v_i, e_i \rangle + \langle v_{M+1, i}, e_{M+1} \rangle $
 is parameters of local clients $i$.

With parameter spaces and updates defined above, we have the following nice \textit{stable decomposition} property.

\begin{definition} [Stable Decomposition]
\label{def:sd}
For a space decomposition, there exists a constant $C_s> 0$, such that for any $v \in \mathcal V$, there exists a decomposition $v = \sum_{i=1}^M v_i$ with $v_i \in \mathcal V_i,$ for $i = 1, \cdots, M$ and 
\begin{equation}
\sum_{i=1}^M \| v_i \|^2 \leq C_s \| v\|_2^2, 
\end{equation} 
where $\| \cdot \|_2$ denote $\ell_2$ norm. 
\end{definition}

By restriction and interpolation operators defined in (\ref{def:res_oper}) and (\ref{def:pro_oper}), we have a stable decomposition of parameter $u_i = I_i R_i \in \V_i$ forms a stable decomposition of $w \in \V$. 

\begin{lemma} \label{lem:s_decomp} Projection $\{ u_i\}_{i \in [M]}$ of $w$ to subspace $\V_i$ defined 
$$
u_i = I_i R_i w
$$
forms a stable decomposition of global parameter $w$, i.e. 
$$
\sum_{i \in [M]} \| u_i \|^2 \leq \| w \|^2. 
$$
\end{lemma}

\subsection{Convergence Results} \label{sec:convergence_res}
To start with, we present the preliminaries, i.e., assumptions and lemma, required for the convergence results.
\begin{assumption} [Local Lipschitz Continuity] 
\label{assump:Lip_local}
The objective function $f$ is continuously differentiable and gradient function of $f$, $\nabla f$, is Lipschitz continuous with Lipschitz constant $L_i > 0$ on subspace $\V_i$, 
\begin{equation} 
\| \nabla f (w) - \nabla f(\bar{w}) \|_2 \leq L_i \| w - \bar{w} \|_2, ~ \forall w, \bar{w} \in \V_i, \notag
\end{equation}
where $\| \cdot \|_2$ denote $\ell_2$ norm. 
\end{assumption}

\begin{assumption} [Strongly Convexity] 
\label{assump:strong_convexity}
The objective function $f$ is strongly convex on space $\V$ if there exists a constant $c > 0$ such that 
\begin{equation}
f(w) \geq f(\bar{w}) + \langle f(\bar{w}), w - \bar{w} \rangle + \frac{1}{2} c \|w -  \bar{w}  \|_2^2, ~~\forall w, \bar{w} \in \V,  \notag
\end{equation}
where $\| \cdot \|_2$ denote $\ell_2$ norm. 
\end{assumption}
\noindent Specifically, if $c = 0$, we have the convexity assumption. 

\begin{assumption} [Convexity] 
\label{assump:convexity}
The objective function $f$ is strongly convex on space $\V$ if 
\begin{equation} 
f(w) \geq f(\bar{w}) + \langle f(\bar{w}), w - \bar{w} \rangle, ~~\forall w, \bar{w} \in \V.  \notag
\end{equation}
\end{assumption}

\begin{assumption} [Bounded Level Set] $f$ is convex and attains its minimum value $f^*$ on a set $S$. There is a finite constant $R_0$ such that the level set of for $f$ defined by $x^0$ is bounded, that is, 
\begin{equation}
\max_{x^* \in S} \max_x \{ \| x - x^* \|: f(x) \leq f(x^0) \} \leq R_0. 
\end{equation}
\end{assumption}

\begin{assumption}[Bounded Shared Gradient Diversity] There exist an uniform upper bound $\lambda$ on the gradient diversity of shared parameter gradients among local objectives, i.e., 
$$
\frac{ \sum_{i=1}^M \frac{1}{M} \| \partial_{w_{M+1}} f_i(v_i) \|^2}{\| \sum_{i=1}^M  \partial_{w_{M+1}} f(\bar{w}) \|^2} \leq \lambda, 
,$$
where 
$f =\frac{1}{M} \sum_{i \in [M]}f_i $, 
$v_i$ is parameter of client $i$ and 
$\bar{w} = \sum_{i \in [M]} I_i R_i v_i$ is virtual average of client parameters.  
\end{assumption}
\noindent For simplicity, we further assume bounded gradients for the shared layers. 

\begin{assumption}[Bounded Shared Gradient] 
There exists an upper bound $B$ on partial gradients of shared parameters among local objectives function, i.e.,
$$
\|  \partial_{w_{M+1}} f(v) 
\| \leq B, 
\forall v \in \mathcal{V}.
$$
\end{assumption}

\subsubsection{Sufficient Decay Property.}
Let $d_{M+1} = \rm{dim} \mathcal V_{M+1}$ be dimensions of shared parameter space $\mathcal V_{M+1}$ and $\left(S\sigma\right)^2 $ be the magnitude of Gaussian noise $\mathbf{n}$ on each dimension. 
Let $L \eqdef \max_{i=1}^M L_i$ be maximum value of local Lipschitz constant. 
We present \textit{sufficient decay} property (Lemma \ref{lem:sd}) to describe how function value decreases between two sequential iterates. This property plays a fundamental role in our convergence proofs. 

\begin{lemma} [Sufficient Decay] 
\label{lem:SD}
 Under the Lipschitz-continuous assumption of gradient function of global objective function $f$,  
let $\{w^t\}$ be sequence generated by Algorithm \texttt{DPFedMTL} with Poisson sampling $\alpha_i=1/M$ and step-size $\eta = \frac{1}{ \lambda L}$, we have 
\begin{eqnarray*}
 \E_{i, n} \left[ f(w^{t+1}) \right] - f(\bar{w}^t)  &\leq& 
 - \frac{1}{2 \lambda M L} \| \nabla f(\bar{w}^t) \|^2  
  + \frac{H^2} {2 \lambda^2 }  B^2  
  + \frac{1}{\lambda^2} \left( 2H^2 + \frac{1}{2ML} \right) d_{M+1} (S \sigma)^2. 
\end{eqnarray*}
where $L = \max_i L_i$, 
$d_{M+1} $ is dimension of subspace $\V_{M+1}$, 
$M$ is number of all clients, 
$H$ is synchronization interval, 
$B$ is upper bound for partial gradient on subspace $\V_{M+1}$, 
$\lambda$ is the bound for gradient diversity, 
$S$ is sensitivity of gradient operation, 
and $S \sigma$ is magnitude of Gaussian noise $n$,
$\bar{w}^{t} = \sum_{i \in [M]} I_iR_i v_i^T$ is virtual average at step $t$,
$w^{t+1}$ is parameter at step $t+1$. 
\label{lem:sd}
\end{lemma}

With the necessary assumptions and preliminaries, we show a neat proof of convergence results of our proposed Algorithm \ref{alg:DPFedMTL} for Lipschitz smooth objective functions under non-convex, convex, and strongly convex cases.

\subsubsection{Convergence for Nonconvex Case.} 
When local objective $f_i$ is nonconvex and first derivate of $f_i$ is Lipschitz continuous, Algorithm \ref{alg:DPFedMTL} converges to a neighborhood of critical point. 
\begin{theorem}[Convergence for Nonconvex Objectives]
Under Assumption of Local Lipschitz Continuity, suppose Algorithm \texttt{DPFedMTL} is run with optimal step-size $\eta = \frac{1}{ \lambda L}$ and Poisson sampling with $\alpha_i=1/M$, the expected sum-of-squares and average-squared gradients of $f$ satisfies the following inequality for all $T \in \N:$
\begin{eqnarray*} 
\E \left[\frac{1}{T}\sum_{t=0}^T \| \nabla f(\bar{w}^T)\|^2 \right]  &\leq& \frac{M H^2 L B^2}{\lambda}   
 + \frac{2ML}{\lambda}\left( 2 H^2 + \frac{1}{2ML} \right) d_{M+1} (S \sigma)^2 
 + \frac{2 \lambda ML}{T}(f(w^0) - f^*) \\
&\rightarrow &   \frac{M H^2 L B^2}{\lambda}    
 + \frac{2ML}{\lambda}\left( 2 H^2 + \frac{1}{2ML} \right) d_{M+1} (S \sigma)^2 \\ 
\end{eqnarray*}
as $T \rightarrow \infty$,
where $w^0$ is initial global parameter,
$f^*$ is optimal value of $f$, 
$L = \max_i L_i$, 
$d_{M+1} $ is dimension of subspace $\V_{M+1}$, 
$M$ is number of all clients, 
$H$ is synchronization interval, 
$B$ is upper bound for partial gradient on subspace $\V_{M+1}$, 
$\lambda$ is the bound for gradient diversity, 
$S$ is sensitivity of gradient operation and $S \sigma$ is magnitude of Gaussian noise on each dimension,
$\bar{w}^T = \sum_{i \in [M]} I_iR_i v_i^T$ is virtual average at step $T$. 
\label{thm:nonconvexity}
\end{theorem}
Theorem \ref{thm:nonconvexity} shows the average of gradients norm converges to a neighbor of $0$ as $T$ goes to infinity, which guarantees the algorithm  converges to a neighbor of critical point in expectation with respect to random variable Gaussian noise $n$ and random selected client index $i$.


\subsubsection{Convergence for Convex Case.} 
When local objective $f_i$ is convex and first derivate of $f_i$ is Lipschitz continuous, Algorithm \ref{alg:DPFedMTL} converges to a neighborhood of optimal value sub-linearly.
\begin{theorem} [Convergence for Convex and Lipschitz Continuous]
Under Assumptions (\ref{assump:Lip_local}) and (\ref{assump:convexity}), 
suppose Algorithm \ref{alg:DPFedMTL} is run with optimal step-size $\eta = \frac{1}{ \lambda L}$, the expected sum-of-squares and average-squared gradients of $f$ satisfies the following inequality for all $T \in \N:$
\begin{equation}
\E_{i,n}[f(\bar{w}^T) ]- f^*
\leq \frac{2 \lambda ML R_0^2}{T} + \alpha , 
\end{equation}
where 
$\alpha = \sqrt{2 \lambda M L R_0^2 \left( \frac{H}{2 \lambda^2} B^2 + \frac{1}{\lambda^2} \left( 2H^2 + \frac{1}{2ML} \right) 
d_{M+1} (S \sigma)^2 \right)} $,
$f^*$ is the optimal value of $f$, 
$L = \max_i L_i$ and $d_{M+1} $ is dimension of subspace $\V_{M+1}$, 
$M$ is number of all clients, 
$H$ is synchronization interval, 
$B$ is upper bound for partial gradient on subspace $\V_{M+1}$, 
$\lambda$ is the bound for gradient diversity, 
$S$ is sensitivity of gradient operation and $\sigma$ is magnitude of Gaussian noise,
$\bar{w}^T = \sum_{i \in [M]} I_iR_i v_i^T$ is virtual average at step $T$. 
\label{thm:convexity}
\end{theorem}

Theorem \ref{thm:convexity} shows Algorithm \ref{alg:DPFedMTL} with fixed step-size $\eta = \frac{1}{ \lambda L}$ and Poisson sampling converges sub-linearly with 
to a neighborhood of optimal value $f^*$ with radius 
$$\alpha =\sqrt{ 2 \lambda M L R_0^2 \left( \frac{H}{2 \lambda^2} B^2 + \frac{1}{\lambda^2} \left( 2H^2 + \frac{1}{2ML} \right) 
d_{M+1} (S \sigma)^2 \right) } $$ in expectation with respect to random variable Gaussian noise $n$ and random selected client index $i$.

\begin{table*}[t!]
\centering
\resizebox{0.98 \textwidth}{!}{
\begin{tabular}{ c|c|c c c c c c c c } \hline
Dataset & Clients 
& Train samples & \multicolumn{3}{c}{Train samples per user}
& Test samples & \multicolumn{3}{c}{Test samples per user} \\
\cline{4-6} \cline{8-10}
	& 	&	&	mean & std. &  skewness & &mean & std. &  skewness \\\hline
CelebA & 	$9343$ & $177,457$	 & $18.99$   & $6.99$  &  $-0.49$
& $22,831$	 & $2.44$   & $0.71$  &  $-0.85$\\ 
FEMNIST & $174$ & $34,711$ & $199.83$ & $75.88$ &  $0.89$ 
& $3,955$ & $22.73$ & $8.44$ &  $0.88$ 
\\\hline
\end{tabular}
}
\caption{Statistics of datasets used for experiments.}
\label{table:stats}
\end{table*}

\subsubsection{Convergence for Strongly Convex Case.} 
When local objective $f_i$ is strongly convex and the first derivate of $f_i$ is Lipschitz continuous, Algorithm \ref{alg:DPFedMTL} converges to a neighborhood of optimal value linearly. 

\begin{theorem} [Convergence for Strongly Convex and Lipschitz Continuous]
Under Assumptions Local Lipschitz Continuity and Global Strong Convexity, suppose Algorithm \texttt{DPFedMTL} is run with optimal step-size $\frac{1}{ \lambda L}$ and Poisson sampling with $\alpha_i=1/M$, the expected optimality gap satisfies the following inequality for all $T \in \N:$
\begin{equation}
\E[f(\bar{w}^{T})] - f^* -  \beta  \leq \left(1 - \frac{c}{ \lambda M L} \right)^{T} \left(f(w^0) - f^* - \beta \right).
\end{equation}
where
 $\beta = \frac{\lambda ML}{c} \left( \frac{H^2}{2 \lambda^2} B^2 + \frac{1}{\lambda^2} \left( 2H^2 + \frac{1}{2ML} \right) 
d_{M+1} (S \sigma)^2 \right)$, 
$w^0$ is initial parameter, 
$f^*$ is optimal value of $f$, 
$L = \max_i L_i$ and $d_{M+1} $ is dimension of subspace $\V_{M+1}$, 
$c$ is the strongly convexity constant, 
$M$ is number of all clients, 
$H$ is synchronization interval, 
$B$ is upper bound for partial gradient on subspace $\V_{M+1}$, 
$\lambda$ is the bound for gradient diversity, 
$S$ is sensitivity of gradient operation and $S\sigma$ is magnitude of Gaussian noise on each dimension,
$c$ is the strongly convexity constant,
$\bar{w}^T = \sum_{i \in [M]} I_iR_i v_i^T$ is virtual average at step $T$. 
\label{thm:strong_convexity}
\end{theorem}

Theorem \ref{thm:strong_convexity} shows Algorithm \ref{alg:DPFedMTL} with fixed step-size $\eta = \frac{1}{ \lambda L}$ and Poisson sampling converges linearly with 
with rate $1 - \frac{c}{ \lambda M L}$ to a neighborhood of optimal value $f^*$ with radius
 $$\beta = \frac{\lambda ML}{c} \left( \frac{H^2}{2 \lambda^2} B^2 + \frac{1}{\lambda^2} \left( 2H^2 + \frac{1}{2ML} \right) 
d_{M+1} (S \sigma)^2 \right)$$
 in expectation with respect to random variable Gaussian noise $n$ and client index $i$.


%
\section{Experiments}
In this section, we conduct experiments on two widely used public datasets to verify the effectiveness of our proposed algorithm, especially on the trade-off between accuracy and privacy under the federated learning setting.
\subsection{Experiment Setup}
\noindent\textbf{Datasets and Models.}
Following the previous FL works \cite{mcmahan2017communication,caldas2018leaf}, two realistic federated benchmark datasets on diverse tasks are used for experiments. Data statistics are shown in Table \ref{table:stats}.
\begin{enumerate}
\item \textbf{CelebA}: 
a Large-scale CelebFaces Dataset~\cite{liu2015faceattributes}. Since the underlying distribution of celebrity images varies, the dataset is non-IID.
It is a classification problem that predicts whether the celebrity in the image is smiling. 
Following the work~\cite{caldas2018leaf}, we use four-layered CNN for the task. 

\item \textbf{Federated Extended MNIST (FEMNIST)}: 
a dataset extending MNIST \cite{cohen2017emnist}. 
The dataset is originally constructed for the federated setting that imposes statistical heterogeneity.
The task is an image classification problem. 
Following the previous work~\cite{caldas2018leaf}, we use two-layered CNN model for the task. 
\end{enumerate}

\noindent\textbf{Baselines, Hyperparameters, and Evaluation.}

$\;\;\;\;\;\;\;\;\;\;\;\;\;$
\noindent\textbf{Baselines.} 
We compare our proposed algorithm \texttt{DPFedMTL} and its variant, i.e., federated version without DP (\texttt{FedMTL}), with several classic baselines, i.e., local model (\texttt{Local}), 
FederatedAveraging (\texttt{FedAvg}~\cite{mcmahan2017communication}), and deferentially private version of FedAvg (\texttt{DPFedAvg}~\cite{heikkila2017differentially, geyer2017differentially, seif2020wireless}) to examine method effectiveness under the non-IID setting.
All the models and algorithms are implemented based on Tensorflow and open-source federated benchmark library LEAF\footnote{LEAF: https://github.com/TalwalkarLab/leaf}.

\noindent\textbf{Hyperparameters.}
Model parameters, such as learning rate and optimizer, all follow the previous benchmark work~\cite{caldas2018leaf} for a fair comparison. 
The client selection at each step follows the Poisson sampling strategy (a uniform distribution $\alpha_i = K/M$ without replacement), where $K$ is the number of clients selected per round and $M$ is the number of all possible clients.
In the experiments, $K$ is set as $20$ and $3$ for CelebA and FEMNIST, respectively. 

\noindent\textbf{Choosing the $\sigma$.}
We choose $\sigma$ w.r.t. $(\epsilon, \delta)$-DP formula in~\cite{abadi2016deep}, i.e.,
$\sigma$=$\sqrt{2 \log(\frac{1.25}{\delta})} / \epsilon$.
Larger $\epsilon$ and $\delta$ give lower privacy guarantees. 
We fix $\delta$ as $1e{-5}$ and vary $\epsilon$ among $\{ 45, 8, 2, 0.5 \}$ to get the corresponding $\sigma$ $\{0.11, 0.65, 2.42, 9.69 \}$.

\noindent\textbf{Evaluation.} All methods are evaluated using the average prediction accuracy over all clients according to the distribution of objective function (\ref{def:global_obj}).  \texttt{DPFedMTL} on datasets CelebA and FEMNIST are performed for 10 runs. Results are evaluated and averaged over all runs.

\begin{table}[t!]
\label{table:overall}
\centering
\begin{tabular}{l|cc}
\hline
Methods & CelebA & FEMNIST  \\ 
\hline
Local              & 89.72\% & 77.91\%  \\ 
FedAvg                   & 89.98\% & 78.44\%  \\ 
FedMTL                     & \textbf{90.24}\% & \textbf{81.24}\%  \\  \hline
DPFedAvg($\sigma$=$0.65$) & 89.50\% & 42.18\%  \\ 
DPFedMTL($\sigma$=$0.65$) & \underline{90.25}\% & \underline{81.58}\%  \\  \hline
DPFedAvg ($\sigma$=$9.69$) & 51.65\% & 0.79\%    \\ 
DPFedMTL($\sigma$=$9.69$) & \underline{89.23}\% & \underline{76.79}\%   \\
\hline
\end{tabular}
\caption{Overall Performance comparison. Results are evaluated in terms of weighted prediction accuracy for all clients.}
\label{tab:overall}
\end{table}
\subsection{Overall Performance} 
We first conduct experiments to examine the effectiveness of the proposed algorithm on two datasets. Results are presented in Table \ref{tab:overall}. First, we find that \texttt{FedMTL} achieves better performance than the other non-private algorithms, i.e., Local and FedAvg, especially for FEMINIST with larger skewness of sample per client.  These results verify the superiority of \texttt{FedMTL} in better capturing the task heterogeneity for non-IID datasets with larger skewness, as shown in Table \ref{table:stats}.
Overall, degraded performances were observed for all deferentially private federated algorithms. Our proposed \texttt{DPFedMTL} outperforms the differentially private version of FedAvg (i.e., \texttt{DPFedAvg}) by a large margin. Moreover, as we increase the privacy protection level $\sigma$ from small ($0.65$) to large ($9.69$), we see a significant model performance drop for \texttt{DPFedAvg}. 
Instead, \texttt{DPFedMTL} shows a strong tolerance towards increasing noise magnitude, which can probably be justified by the fact that the shared lower layers are more transferable and insensitive to noises, while task-specific top layers do help to improve the task performance.

Furthermore, our theoretical results (presented in Theorem \ref{thm:nonconvexity} \ref{thm:convexity} and \ref{thm:strong_convexity}) show that Algorithm \ref{alg:DPFedMTL} converges to a neighbor of optimal value or critical point. As the magnitude of $\sigma$ increases, the radius of the neighbor becomes large, which exert obstacles for improving the accuracy. Thus, it is important to balance the trade-off between privacy and utility.

\begin{table}[t!]
\centering
\begin{tabular}{l|cc|cc}
\hline
 & \multicolumn{2}{c|}{CelebA (T=400)} & \multicolumn{2}{c}{FEMINST (T=1000)} \\ \cline{2-5} 
 & FedAvg & FedMTL & FedAvg & FedMTL  \\ \hline
No DP  & 89.98 \% & 90.24\% & 78.44\%  & 81.24\%   \\ \hline
DP($\sigma$=0.11) & 90.05\% & \textbf{90.28}\% &76.00\%  &\textbf{81.56}\%  \\
DP($\sigma$=0.65) & 89.50\% & \textbf{90.25}\% &42.18\%  &\textbf{81.58}\% \\
DP($\sigma$=2.42)  & 85.28\% & \textbf{90.26}\% &5.67\%   &\textbf{80.60}\%  \\
DP($\sigma$=9.69)  & 51.65\% & \textbf{89.23}\% &0.79\%   &\textbf{76.79}\% \\ \hline
\end{tabular}
\caption{Performance comparison under different noise levels.}
\label{tab:tradeoff}
\end{table}
\subsection{Privacy Analysis}
\noindent\textbf{Privacy v.s. Utility.}
In this experiment, we analyze the trade-offs between the privacy and utility for the proposed  \texttt{DPFedMTL} against \texttt{DPFedAvg} on different datasets by varying the $\sigma$.
As we can see from Table~\ref{tab:tradeoff}, when increasing the $\sigma$ from $0.11$ to $9.69$, \texttt{DPFedMTL} shows comparable performance or minor drop in performance on both datasets. However, for \texttt{DPFedAvg}, increasing the $\sigma$ leads to significant deterioration of the model utility, especially for FEMINST.
Note, each experiment result for our proposed \texttt{DPFedMTL} is averaged over 10 runs. 
In summary, our proposed \texttt{DPFedMTL} is robust and significantly outperforms \texttt{DPFedAvg} with the increasing noise level $\sigma$. 
\noindent\textbf{Privacy Accounting.}
For DP types of algorithms, another important issue is computing the overall privacy loss of the training, since privacy loss accumulates along with each training step as we have discussed in Section \ref{sec:privacy}.
We employ the notion of GDP and CLT
for privacy accounting, 
as it has been proven in the recent literature \cite{dong2019gaussian} that bounds for the best possible approximation via an $(\epsilon, \delta)$-DP
guarantee (e.g., moments accountant) is substantially looser than the CLT approximated bound.
More details of privacy accounting can be found in Section \ref{sec:privacy}. 
\begin{figure}[t!]
\begin{subfigure}{.45\textwidth}
  \centering
  \includegraphics[width=.98\linewidth]{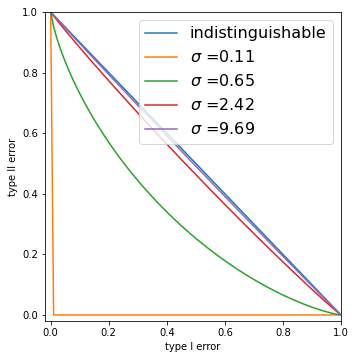}  
  \caption{CelebA (T=400)}
  \label{fig:sub-second}
\end{subfigure}
\begin{subfigure}{.45\textwidth}
  \centering
  \includegraphics[width=.98\linewidth]{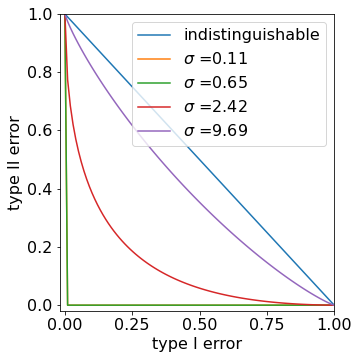}  
  \caption{FEMNIST (T=1000)}
  \label{fig:sub-first}
\end{subfigure}
\caption{Privacy guarantees for the proposed DPFedMTL after privacy composition during the whole course of training.
Privacy analysis is performed for different $\sigma$ settings and privacy loss is approximated by Central Limit Theorem.
}
\label{fig:privacy}
\end{figure}

Figure~\ref{fig:privacy} plots the distinguishability between $M(S)$ and $M(S')$ in terms of type I and type II errors based on CLT based approximation.
As presented in Figure~\ref{fig:privacy}, 
the blue line denotes the situation where no information can be inferred, 
while other colored lines represent the accumulated privacy guarantees with different $\sigma$ under the privacy composition after the whole training process. The closer the lines to the blue line (i.e., indistinguishable), the better the privacy guarantee it is after composition.
Here, experiments for CelebA and FEMINST are trained for 400 and 1000 steps, respectively, for all $\sigma$ settings.
We observe after privacy composition, a small $\sigma=0.65$ is enough for CelebA to achieve good privacy guarantee with good utility (accuracy $90.25\%$), while FEMINIST needs much larger $\sigma=9.69$ to provide good privacy protection with somewhat degraded utility (accuracy $76.79\%$, i.e., reduced by -$4.45\%$). 
In all, dataset CelebA is more sensitive to noise added compared to dataset FEMINST.

\subsection{Algorithm Convergence}
Furthermore, we compare the convergence of different methods in Figure \ref{fig:convergence}.
Following the previous privacy analysis, we present the convergence for DP types of methods under the setting with noise level $\sigma$ set as $0.65$ for CelebA and $9.69$ for FEMNIST, i.e., the parameters with good privacy and utility trade-off.
From this figure, we observe that both \texttt{FedMTL} (i.e., green line) and  \texttt{DPFedMTL} (i.e., red line) converge much faster than their \texttt{FedAvg} based counterparts, while achieve comparable convergence speed compared to \texttt{Local} models.
\begin{figure}[ht]
\begin{subfigure}{.23\textwidth}
  \centering
  \includegraphics[width=.98\linewidth]{pic/celeba/sigma_9.68961}  
  \caption{CelebA with $\sigma$=$9.69$} 
  \label{fig:sub-second}
\end{subfigure}
\begin{subfigure}{.23\textwidth}
  \centering
  \includegraphics[width=.98\linewidth]{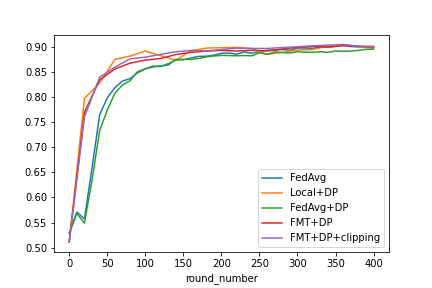}  
  \caption{CelebA with $\sigma$=$9.69$} 
  \label{fig:sub-second}
\end{subfigure}
\begin{subfigure}{.23\textwidth}
  \centering
  \includegraphics[width=.98\linewidth]{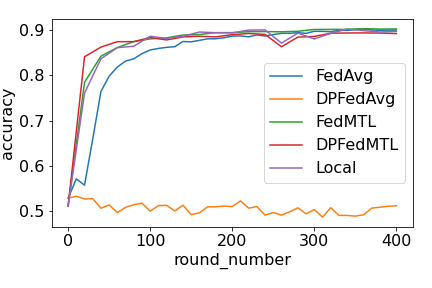}  
  \caption{FEMINST with $\sigma$=$0.65$}
  \label{fig:sub-second}
\end{subfigure}
\begin{subfigure}{.23\textwidth}
  \centering
  \includegraphics[width=.98\linewidth]{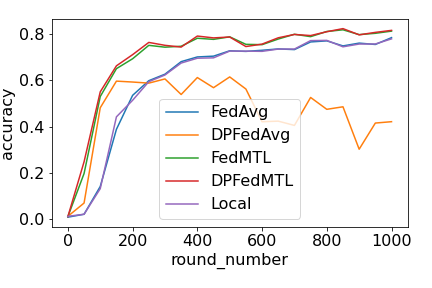}  
  \caption{FEMINST with $\sigma$=$0.65$}
  \label{fig:sub-second}
\end{subfigure}
\caption{Convergence comparison of different methods.}
\label{fig:convergence}
\end{figure}

\section{Conclusion}
In this paper, we present a differentially private federated multi-task learning algorithm. The goal is to protect privacy from the client level and balance the trade-off between privacy and utility guarantees. Our algorithm separates networks to shared layers and task-specific layers, which solves the inherent data heterogeneity in federated learning.
We provide privacy analysis based on Gaussian DP and convergence analysis using subspace decomposition for the cases with Lipschitz smooth objective functions under16the non-convex, convex, and strongly convex settings. Quantitative evaluations over two public datasets demonstrate the effectiveness of our method.

\bibliographystyle{plain}
\bibliography{dp}

\end{document}